\documentclass[10pt,twocolumn,letterpaper]{article}

\usepackage{cvpr}
\usepackage{times}
\usepackage{epsfig}
\usepackage{graphicx}
\usepackage{amsmath}
\usepackage{amssymb}

\usepackage{booktabs}       
\usepackage{amsfonts}  
\usepackage{amsmath}
\usepackage{nicefrac}       
\usepackage{microtype}      
\usepackage{xcolor}
\usepackage{authblk}

\usepackage{booktabs}       
\usepackage{epsfig}         
\usepackage{float}          
\usepackage{graphicx}       

\usepackage{microtype}      
\usepackage{subfig}
\usepackage{url}            
\usepackage{xspace}

\usepackage{multirow}
\usepackage{enumitem}
 \usepackage[ruled,vlined,noresetcount]{algorithm2e}
 \usepackage{algpseudocode}
 \usepackage{amsthm}
 
\usepackage{wrapfig}
\usepackage{bm}

\newtheorem{Thm}{Theorem}
\newtheorem*{Thm*}{Theorem}

\usepackage[pagebackref=true,breaklinks=true,letterpaper=true,colorlinks,bookmarks=false]{hyperref}

\cvprfinalcopy 

\def\cvprPaperID{6102} 

\ifcvprfinal\fi
\begin{document}

\title{The Secret Revealer: Generative Model-Inversion Attacks Against Deep Neural Networks}

\author[1]{Yuheng Zhang \thanks{Both authors contributed equally}}
\author[2]{Ruoxi Jia $^*$}
\author[1]{Hengzhi Pei}
\author[3]{Wenxiao Wang}
\author[4]{Bo Li}
\author[2]{Dawn Song}
\affil[1]{Fudan University} \affil[2]{University of California at Berkeley}
\affil[3]{Tsinghua University}
\affil[4]{University of Illinois at Urbana-Champaign}
\affil[ ]{\textit {\{yuhengzhang16,hzpei16\}@fudan.edu.cn, ruoxijia@berkeley.edu, wangwx16@mails.tsinghua.edu.cn, lxbosky@gmail.com, dawnsong@gmail.com}}

\maketitle

\begin{abstract}
This paper studies model-inversion attacks, in which the access to a model is abused to infer information about the training data. Since its first introduction by~\cite{fredrikson2014privacy}, such attacks have raised serious concerns given that training data usually contain privacy-sensitive information. Thus far, successful model-inversion attacks have only been demonstrated on simple models, such as linear regression and logistic regression. Previous attempts to invert neural networks, even the ones with simple architectures, have failed to produce convincing results. We present a novel attack method, termed the generative model-inversion attack, which can invert deep neural networks with high success rates. Rather than reconstructing private training data from scratch, we leverage partial public information, which can be very generic, to learn a distributional prior via generative adversarial networks (GANs) and use it to guide the inversion process. Moreover, we theoretically prove that a model's predictive power and its vulnerability to inversion attacks are indeed two sides of the same coin---highly predictive models are able to establish a strong correlation between features and labels, which coincides exactly with what an adversary exploits to mount the attacks. Our extensive experiments demonstrate that the proposed attack improves identification accuracy over the existing work by about $75\%$ for reconstructing face images from a state-of-the-art face recognition classifier. We also show that differential privacy, in its canonical form, is of little avail to defend against our attacks. 
\end{abstract}

\vspace{-0.5em}
\section{Introduction}
\vspace{-0.5em}
Deep neural networks (DNNs) have been adopted in a wide range of applications, including computer vision, speech recognition, healthcare, among others. The fact that many compelling applications of DNNs involve processing sensitive and proprietary datasets raised great concerns about privacy. In particular, when machine learning (ML) algorithms are applied to private training data, the resulting models may unintentionally leak information about training data through their output (i.e., black-box attack) or their parameters (i.e., white-box attack).

A concrete example of privacy attacks is model-inversion (MI) attacks, which aim to reconstruct sensitive features of training data by taking
advantage of their correlation with the
model output. Algorithmically, MI attacks are implemented as an optimization problem seeking for the sensitive feature value that achieves the maximum likelihood under the target model. The first MI attack was proposed in the context of genomic privacy~\cite{fredrikson2014privacy}, where the authors showed that adversarial access to a linear regression model for personalized medicine can be abused to infer private genomic attributes about individuals in the training dataset. Recent work~\cite{fredrikson2015model} extended MI attacks to other settings, e.g., recovering an image of a person from a face recognition model given just their name, and other target models, e.g., logistic regression and decision trees.

Thus far, effective MI attacks have only been demonstrated on the aforementioned simple models. It remains an open question whether it is possible to launch the attacks against a DNN and reconstruct its private training data. The challenges of inverting DNNs arise from the intractability and ill-posedness of the underlying attack optimization problem. For neural networks, even the ones with one hidden layer, the corresponding attack optimization becomes a non-convex problem; solving it via gradient descent methods may easily stuck in local minima, which leads to poor attack performance. Moreover, in the attack scenarios where the target model is a DNN (e.g., attacking face recognition models), the sensitive features (face images) to be recovered often lie in a high-dimensional, continuous data space. Directly optimizing over the high-dimensional space without any constraints may generate unrealistic features lacking semantic information.

  
   


In this paper, we focus on image data and propose a simple yet effective attack method, termed the generative model-inversion (GMI) attack, which can invert DNNs and synthesize private training data with high fidelity. The key observation supporting our approach is that it is arguably easy to obtain information about the general data distribution, especially for the image case. For example, against a face recognition classifier, the adversary could randomly crawl facial images from the Internet without knowing the private training data. We find these datasets, although may not contain the target individuals, still provide rich knowledge about how a face image might be structured; extraction and proper formulation of such prior knowledge will help regularize the originally ill-posed inversion problem. 
We also move beyond specific attack algorithms and explore the fundamental reasons for a model's susceptibility to inversion attacks.
We show that the vulnerability is unavoidable for highly predictive models, since these models are able to establish a strong correlation between features  and  labels, which coincides exactly with what an adversary exploits to mount MI attacks.

Our contributions can be summarized as follows: (1) We propose to use generative models to learn an informative prior from public datasets so as to regularize the ill-posed inversion problem. (2) We propose an end-to-end GMI attack algorithm based on GANs, which can reveal private training data of DNNs with high fidelity. (3) We present a theoretical result that uncovers the fundamental connection between a model's predictive power and its susceptibility to general MI
attacks and empirically validate it. (4) We conduct extensive experiments to demonstrate the performance of the proposed attack. (5) We show that differential privacy, a ``gold standard'' privacy notion nowadays, is of little avail to protect against our attacks, because it does not explicitly aim to protect the secrecy of attributes in training data. This raises the question: What is the right notion for attribute privacy? Answering this question is an important future work.

\vspace{-0.5em}
\section{Related Work}
\vspace{-0.5em}
Privacy attacks against ML models consist of methods that aim to reveal some aspects of training data. Of particular interest are membership attacks and MI attacks. Membership attacks aim to determine whether a given individual's data is used in training the model~\cite{shokri2017membership}. MI attacks, on the other hand, aim to reconstruct the features corresponding to specific target labels.


In parallel to the emergence of various privacy attack methods, there is a line work that formalizes the privacy notion and develops defenses with formal and provable privacy guarantees. One dominate definition of privacy is differential privacy (DP), which carefully randomizes an algorithm so that its output does not to depend too much on any individuals’ data~\cite{dwork2014algorithmic}. In the context of ML algorithms, DP guarantees protect against attempts to infer whether a data record is included in the training set from the trained model~\cite{abadi2016deep}. By definition, DP limits the success rate of membership attacks. However, it does not explicitly protect attribute privacy, which is the target of MI attacks~\cite{fredrikson2014privacy}.

The first MI attack was demonstrated in~\cite{fredrikson2014privacy}, where the authors presented an algorithm to recover genetic markers given the linear regression that uses them as input features, the response of the model, as well as other non-sensitive features of the input. \cite{hidano2017model} proposed a algorithm that allows MI attacks to be carried out without the knowledge of non-sensitive features by poisoning training data properly. 
Despite the generality of the algorithmic frameworks proposed in the above two papers, the evaluation of the attacks is only limited to linear models. \cite{fredrikson2015model} discussed the application of MI attacks to more complex models including some shallow neural networks in the context of face recognition. Although the attack can reconstruct face images with identification rates much higher than random guessing, the recovered faces are indeed blurry and hardly recognizable. Moreover, the quality of reconstruction tends to degrade for more complex architectures. \cite{yang2019adversarial} proposed to train a separate network that swaps the input and output of the target network to perform MI attacks. The inversion model can be trained with black-box accesses to the target model. However, their approach cannot directly be benefited from the white-box setting.

Moreover, several recent papers started to formalize MI attacks and study the factors that affect a model's vulnerability from a theoretical viewpoint. For instance, \cite{wu2016methodology} characterized model invertibility for Boolean functions using the concept of influence from Boolean analysis; \cite{yeom2018privacy} formalized the risk that the model poses specifically to individuals in the training data and shows that the risk increases with the degree of overfitting of the model. However, their theory assumed that the adversary has access to the join distribution of private feature and label, which is overly strong for many attack scenarios. Our theory does not rely on this assumption and better supports the experimental findings.

The algorithms of MI attacks resemble an orthogonal line of work on feature visualization~\cite{nguyen2016synthesizing,yosinski2015understanding}, which also attempts to reconstruct an image that maximally activates a target network. Our work differs from the existing work on feature visualization in that the proposed algorithm adopts a novel optimization objective which results in more realistic image recovery and can incorporate possible auxiliary knowledge of the attacker.

\vspace{-0.5em}
\section{Generative MI Attack}
\vspace{-0.5em}
An overview of our GMI attack is illustrated in Figure~\ref{fig:overview}. In this section, we will first discuss the threat model and then present our attack method in details.

\begin{figure}[t]
  \centering
  \includegraphics[width=\linewidth]{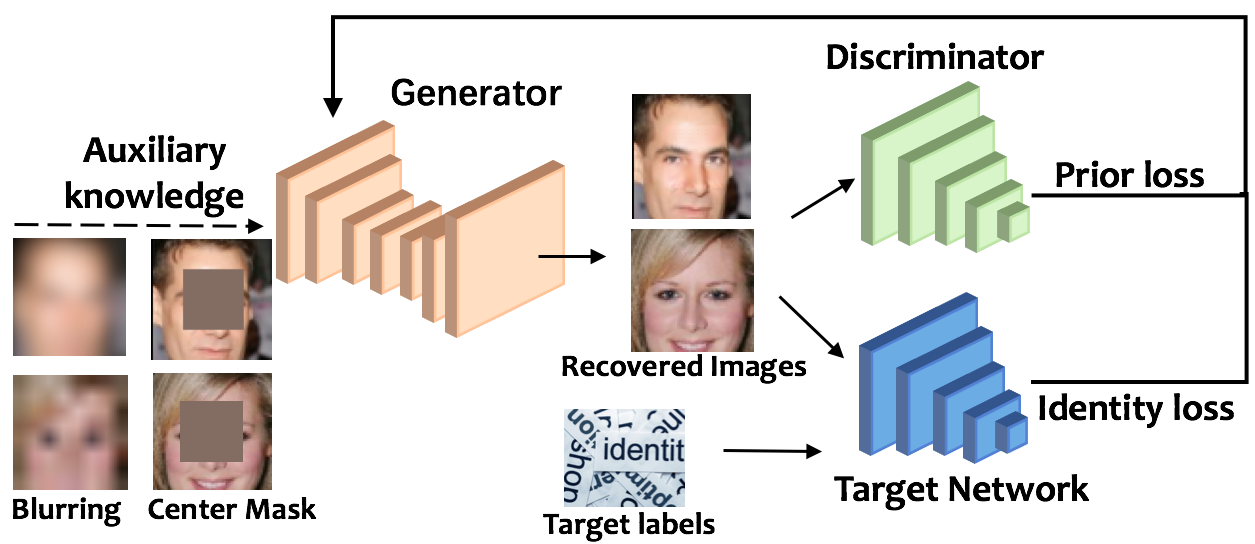}
    \vspace{-2em}
  \caption{Overview of the proposed GMI attack method.}
  \vspace{-1.5em}
  \label{fig:overview}
\end{figure}

\vspace{-0.5em}
\subsection{Threat Model}
\vspace{-0.5em}
In traditional MI attacks, an adversary, given a model trained to predict specific labels, uses it to make predictions of sensitive features used during training. Throughout the paper, we will refer to the model subject to attacks as the \emph{target network}. We will focus on the \emph{white-box} setting, where the adversary is assumed to have access to the target network $f$ and employs some inference technique to discover the features $x$ associated with a specific label $y$. In addition to $f$, the adversary may also have access to some auxiliary knowledge that facilitates his/her inference. We will use face recognition classifiers as a running example for the target network. Face recognition classifiers label an image containing a face with an identifier corresponding to the individual depicted in the image. The corresponding attack goal is to recover the face image for some specific identity based on the target classifier parameters.
\vspace{-1.2em}
\paragraph{Possible Auxiliary Knowledge.} Examples of auxiliary knowledge could be a corrupted image which only contains nonsenstive information, such as background pixels in a face image, or a blurred image. This auxiliary knowledge might be easy to obtain, as blurring and corruption are often applied to protect anonymity of individuals in public datasets~\cite{carrell2012hiding,li2019hideme}.


\vspace{-1.2em}
\paragraph{Connection to Image Inpainting.} The setup of MI attacks on images resembles the widely studied image inpainting tasks in computer vision, which also try to fill missing pixels of an image. The difference is, however, in the goal of the two. MI attacks try to fill the sensitive features associated with a specific identity in the training set. In contrast, image inpainting tasks only aim to synthesize visually realistic and semantically plausible pixels for the missing regions; whether the synthesized pixels are consistent with a specific identity is beyond the scope. Despite the difference, our approach to MI attacks leverages some training strategies from the venerable line of work on image inpainting~\cite{yeh2017semantic,iizuka2017globally,yang2019diversity} and significantly improves the realism of the reconstructed images over the existing attack methods.


\vspace{-0.5em}
\subsection{Inferring Missing Sensitive Features}
\vspace{-0.5em}
\label{section:pipeline}
To realistically reconstruct missing sensitive regions in an image, our approach utilizes the generator $G$ and the discriminator $D$, all of which are trained with public data. After training, we aim to find the latent vector $\hat{z}$ that achieves highest likelihood under the target network while being constrained to the data manifold learned by $G$. However, if not properly designed, the generator may not allow the target network to easily distinguish between different latent vectors. For instance, in extreme cases, if the generated images of all latent vectors collapse to the same point in the feature space of the target network, then there is no hope to identify which one is more likely to appear in its private training set of the target network. To address this issue, we present a simple yet effective loss term to promote the diversity of the data manifold learned by $G$ when projected to the target network's feature space.

Specifically, our reconstruction process consists of two stages: (1) \emph{Public knowledge distillation}, in which we train the generator and the discriminators on public datasets in order to encourage the generator to generate realistic-looking images. The public datasets can be unlabeled and have no identity overlapping with the private dataset. (2) \emph{Secret revelation}, in which we make use of the generator obtained from the first stage and solve an optimization problem to recover the missing sensitive regions in an image.

For the first stage, we leverage the canonical Wasserstein-GAN~\cite{arjovsky2017wasserstein} training loss:
\begin{equation}
\begin{aligned}
\min_G\max_{D} L_\text{wgan}(G,D) = E_x[D(x)] - E_z[D(G(z))]
\end{aligned}
\end{equation}
When the auxiliary knowledge (e.g., blurred or corrupted version of the private image) is available to the attacker, we let the generator take the auxiliary knowledge as an additional input. Moreover, when the extra knowledge is a corrupted image, we adopt two discriminators for discerning whether an image is real or artificial, like~\cite{iizuka2017globally}. The global discriminator looks at the reconstructed image to assess if it is coherent as a whole, while the local discriminator looks only at a randomly selected patch containing the mask boundary to ensure the local consistency of the generated patches at the boundary area.
However, different from~\cite{iizuka2017globally} which fuses the outputs of the two discriminators together by a concatenation layer that predicts a value corresponding to the
probability of the image being real, we allow two discriminators to have separate outputs, as we find it make the training loss converge faster empirically. The detailed architecture of the GAN is presented in the supplementary material.

In addition, inspired by~\cite{yang2019diversity}, we introduce a diversity loss term that promotes the diversity of the images synthesized by $G$ when projected to the target network's feature space. Let $F$ denote the feature extractor of the target network. The diversity loss can thus be expressed as
\begin{align}
\label{eqn:diversity_loss}
    \max_G L_\text{div}(G) = E_{\mathbf{z_1},\mathbf{z_2}}\bigg[\frac{\left\|F(G(\mathbf{z_1}))-F(G(\mathbf{z_2}))\right\|}{\left\|\mathbf{z_1} - \mathbf{z_2}\right\|}\bigg]
\end{align}
As discussed above, larger diversity will facilitate the targeted network to discern the generated image that is most likely to appear in its private training set. Our full objective for public knowledge distillation can be written as $\min_G\max_{D} L_\text{wgan}(G,D) - \lambda_d L_\text{div}(G)$.

In the secret revelation stage, we solve the following optimization to find the latent vector that generates an image achieving the maximum likelihood under the target network while remaining realistic: $\hat{z} = \mathop{\arg\min}_z L_\text{prior}(z) + \lambda_i L_\text{id}(z)$,
where the prior loss $L_\text{prior}(z)$ penalizes unrealistic images and the identity loss $L_\text{id}(z)$ encourages the generated images to have high likelihood under the targeted network. They are defined, respectively, by
\begin{align}
   & L_\text{prior}(z) = - D(G(z))
    \label{eqn:identity_loss}
   & L_\text{id}(z) = - \log[C(G(z))]
\end{align}
where $C(G(z))$ represents the probability of $G(z)$ output by the target network. 









\vspace{-0.5em}
\section{Connection Between Model Predictive Power and MI Attacks}
\vspace{-0.5em}



For a fixed data point $(x,y)$, we can measure the performance of a model $f$ for predicting the label $y$ of feature $x$ using the log likelihood $\log p_f(y|x)$. It is known that maximizing the log likelihood is equivalent to minimizing the cross entropy loss---one of the most commonly used loss functions for training DNNs. Thus, throughout the following analysis, we will focus on the log likelihood as a model performance measure.

Now, suppose that $(X,Y)$ is drawn from an unknown data distribution $p(X,Y)$. Moreover, $X=(X_s,X_{ns})$, where $X_s$ and $X_{ns}$ denote the sensitive and non-sensitive part of the feature, respectively. We can define the predictive power of the sensitive feature $X_s$ under the model $f$ (or equivalently, the predictive power of model $f$ using $X_s$) as the change of model performance when excluding it from the input, i.e., $E_{(X,Y)\sim p(X,Y)}[\log p_f(Y|X_s,X_{ns}) - \log p_f(Y|X_{ns})]$. Similarly, we define the predictive power of the sensitive feature given a specific class $y$ and nonsensitive feature $x_{ns}$ as
\begin{equation}
\begin{aligned}
    U_f(x_{ns},y) =  E_{X_s\sim p(X_s|y,x_{ns})} [\log& p_{f}(y|X_s, x_{ns}) \\
    &- \log p_{f} (y|x_{ns})]
\end{aligned}
\end{equation}



We now consider the measure for the MI attack performance. Recall the goal of the adversary is to guess the value of $x_s$ given its corresponding label $y$, the model $f$, and some auxiliary knowledge $x_{ns}$.
The best attack outcome is the recovery of the entire posterior distribution of the sensitive feature, i.e., $p(X_s|y,x_{ns})$. However, due to the incompleteness of the information available to the adversary, the best possible attack result that adversary can achieve under the attack model can be captured by $p_f(X_s|y,x_{ns}) \propto p_f(y|X_s,x_{ns})p(X_s|x_{ns})$, assuming that the adversary can have a fairly good estimate of $p(X_s|x_{ns})$. Such estimate can be obtained by, for example, learning from public datasets using the method in Section~\ref{section:pipeline}. Although MI attack algorithms often output a single feature vector as the attack result, these algorithms can be adapted to output a feature distribution instead of a single point by randomizing the starting guess of the feature. Thus, it is natural to measure the MI attack performance in terms of the similarity between $p(X_s|y,x_{ns})$ and $p_f(X_s|y,x_{ns})$. The next theorem indicates that the vulnerability to MI attacks is unavoidable if the sensitive features are highly predictive under the model. When stating the theorem, we use the negative KL-divergence $S_\text{KL}(\cdot||\cdot)$ to measure the similarity between two distributions.

\begin{Thm}
Let $f_1$ and $f_2$ be two models such that for any fixed label $y\in \mathcal{Y}$, $ U_{f_1}(x_{ns},y)\geq U_{f_2}(x_{ns},y)$. Then, $S_\text{KL}(p(X_s|y,x_{ns})||p_{f_1}(X_s|y,x_{ns})) \geq  S_\text{KL}(p(X_s|y,x_{ns})||p_{f_2}(X_s|y,x_{ns}))$.
\end{Thm}
We omit the proof of the theorem to the supplementary material.
Intuitively, highly predictive models are able to build a strong correlation between features and labels, which coincides exactly with what an adversary exploits to launch MI attacks; hence, more predictive power inevitably leads to higher attack performance.

In~\cite{yeom2018privacy}, it is argued that a model is more vulnerable to MI attacks if it overfits data to a greater degree. Their result is seemingly contradictory with ours, because fixing the training performance, more overfitting implies that the model has less predictive power. However, the assumption underlying their result is fundamentally different from ours, which leads to the disparities. The result in~\cite{yeom2018privacy} assumes that the adversary has access to the joint distribution $p(X_s,X_{ns},Y)$ that the private training data is drawn from and their setup of the goal of the MI attack is to learn the sensitive feature associated with a given label in a specific training dataset. By contrast, our formulation of MI attacks is to learn about private feature distribution $p(X_s|y,x_{ns})$ for a given label $y$ from the model parameters. We do not assume that the adversary has the prior knowledge of $p(X_s,X_{ns},Y)$, as it is a overly strong assumption for our formulation---the adversary can easily obtain $p(X_s|y,x_{ns})$ for any labels and any values of non-sensitive features when having access to the joint distribution.

\vspace{-0.5em}
\section{Experiments}

\subsection{Experimental Setup}
\vspace{-0.5em}

\paragraph{Dataset.} We evalaute our method using three datasets: (1) the MNIST handwritten digit data~\cite{Lecun1998Gradient-Based} (\texttt{MNIST}), (2) the Chest X-ray Database~\cite{wang2017chestx} (\texttt{ChestX-ray8}), and (3) the CelebFaces Attributes Dataset~\cite{liu2015deep} (\texttt{CelebA}) containing 202,599 face images of 10,177 identities with coarse alignment. We crop the images at the center and resize them to 64$\times$64 so as to remove most background. 

\vspace{-1.2em}\paragraph{Protocol.} We split each dataset into two disjoint parts: one part used as the private dataset to train the target network and the other as a public dataset for prior knowledge distillation. \emph{The public data, throughout the experiments, do not have class intersection with the private training data of the target network.} Therefore, the public dataset in our experiment only helps the adversary to gain knowledge about features generic to all classes and does not provide information about private, class-specific features for training the target network. This ensures the fairness of the comparison with the existing MI attack~\cite{fredrikson2015model}.

\vspace{-1.2em}\paragraph{Models.} We implement several different target networks with varied complexities. Some of the networks are adapted from existing ones by adjusting the number of outputs of their last fully connected layer to our tasks. For the digit classification on \texttt{MNIST}, our target network consists of 3 convolutional layers and 2 pooling layers. For the disease prediction on \texttt{ChestX-ray8}, we use ResNet-18 adapted from~\cite{He2015DeepRL}. For the face recognition tasks on \texttt{CelebA}, we use the following networks: (1) \texttt{VGG16} adapted from~\cite{Simonyan2014Very}; (2) \texttt{ResNet-152} adapted from ~\cite{He2015DeepRL}; (3) \texttt{face.eoLVe} adapted from the state-of-the-art face recognition network~\cite{cheng2017know}.

\vspace{-1.2em}\paragraph{Training.} We split the private dataset defined above into training set (90\%) and test set (10\%) and use the SGD optimizer with learning rate $10^{-2}$, batch size $64$, momentum $0.9$ and weight decay $10^{-4}$ to train these networks. To train the GAN in the first stage of our attack pipeline, we set $\lambda_d=0.5$ and use the Adam optimizer with the learning rate $0.004$, batch size $64$, $\beta_1=0.5$, and $\beta_2=0.999$~\cite{kingma2014adam}. In the second stage, we set $\lambda_i=100$ and use the SGD optimizer to optimize the latent vector $z$ with the learning rate $0.02$, batch size $64$ and momentum $0.9$. $z$ is drawn from a zero-mean unit-variance Gaussian distribution. We randomly initialize $z$ for 5 times and optimize each round for 1500 iterations. We choose the solution with the lowest identity loss as our final latent vector.


\vspace{-0.5em}
\subsection{Evaluation Metrics} 
\vspace{-0.5em}
Evaluating the success of MI attacks requires to assess whether the recovered image exposes the private information about a target label. Previous works analyzed the attack performance mainly qualitatively by visual inspection. Herein, we introduce four metrics which allow for quantitatively assessing the MI attack efficacy and performing evaluation at a large scale. 

\vspace{-1.2em}\paragraph{Peak Signal-to-Noise Ratio (PSNR).} PSNR is the ratio of an image's maximum squared pixel fluctuation over the mean squared error between the target image and the reconstructed image~\cite{hore2010image}. PSNR measures the pixel-wise similarity between two images. The higher the PSNR, the better the quality of the reconstructed image. 

However, oftentimes, the reconstructed image may still reveal identity information even though it is not close to the target image pixel-wise. For instance, a recovered face with different translation, scale and rotation from the target image will still incur privacy loss. This necessitates the need for the following metrics that can evaluate the similarity between the reconstructed and the target image at a semantic level. 

\vspace{-1.2em}\paragraph{Attack Accuracy (Attack Acc).} 
We build an \emph{evaluation classifier} that predicts the identity based on the input reconstructed image. If the evaluation classifier achieves high accuracy, the reconstructed image is considered to expose private information about the target label. The evaluation classifier should be different from the target network because the reconstructed images may incorporate features that overfit the target network while being semantically meaningless. Moreover, the evaluation classifier should be highly performant. For the reasons above, we adopt the state-of-the-art architecture in each task as the evaluation classifier. For \texttt{MNIST}, our evaluation network consists of 5 convolutional layers and 2 pooling layers. For \texttt{ChestX-ray8}, our evaluation network is adapted from VGG-19 ~\cite{Simonyan2014Very}. For \texttt{CeleA}, we use the model in~\cite{cheng2017know} as the evaluation classifier. We first pretrain it on the \texttt{MS-Celeb-1M}~\cite{guo2016ms} and then fine tune on the identities in the training set of the target network. The resulting evaluation classifier can achieve $96\%$ accuracy on these identities.

\vspace{-1.2em}\paragraph{Feature Distance (Feat Dist).} Feat Dist measures the $l_2$ feature distance between the reconstructed image and the centroid of the target class. The feature space is taken to be the output of the penultimate layer of the evaluation network. 

\vspace{-1.2em}\paragraph{K-Nearest Neighbor Distance (KNN Dist).} KNN Dist looks at the shortest distance from the reconstructed image to the target class. We identify the closest data point to the reconstructed image in the training set and output their distance. The distance is measured by the $l_2$ distance between the two points in the feature space of the evaluation classifier.


\vspace{-0.5em}
\subsection{Experimental Results}
\vspace{-0.5em}

We compare our proposed GMI attack with the existing model-inversion attack (EMI), which implements the algorithm in~\cite{fredrikson2015model}. In this algorithm, the adversary only exploits the identity loss for image reconstruction and returns the pixel values that minimize the the identity loss. 
Another baseline is pure image inpainting (PII). PII minimizes the W-GAN loss and performs image recovery based purely on the information completely from the public dataset. The comparison with PII will indicate whether our attack truly reveals private information or merely learns to output a realistic-looking image. The network architectures for PII are exhibited in the the supplementary material.

\vspace{-0.5em}
\subsubsection{Attacking Face Recognition Classifiers}
\vspace{-0.5em}

For \texttt{CelebA}, the private set comprises 30,000 images of 1000 identities and samples from the rest are used as a public dataset. We evaluate the attack performance in the three settings: (1) the attacker does not have any auxiliary knowledge about the private image, in which case he/she will recover the image from scratch; (2) the attacker has access to a blurred version of the private image and his/her goal is to deblur the image; (3) the attacker has access to a corrupted version of the private image wherein the sensitive, identity-revealing features (e.g., nose, mouth, etc) are blocked. 

Table~\ref{table:without} compares the performance of our proposed GMI attack against EMI and PII for different network architectures. We can see that EMI can hardly attack deep nets and achieves around zero attack accuracy. Since EMI does not exploit any prior information, the inversion optimization problem is extremely ill-posed and performing gradient descent ends up at some visually meaningless local minimum, as illustrated by Figure~\ref{fig:img_results}.
Interestingly, despite having the meaningless patterns, these images can all be classified correctly into the target label by the target network. Hence, \emph{ the existing MI attack tends to generate ``adversarial examples''~\cite{goodfellow2014explaining} that can fool the target network but does not exhibit any recognizable features of the private data.} GMI is much more effective than EMI. Particularly, our method improves the accuracy of the attack against the state-of-the-art \texttt{face.evoLVe} classifier over EMI by 75\% in terms of Top-5 attack accuracy. Also, note that models that are more sophisticated and have more predictive power are more susceptible to attacks. We will examine this phenomenon in more details in Section~\ref{section:exp_power}. 

Figure~\ref{fig:img_results} also compares GMI with PII, which synthesizes a face image completely based on the information from the public dataset. We can see that although PII leads to realistic recoveries, the reconstructed images do not present the same identity features as the target images. This can be further corroborated by the quantitative results in Table~\ref{table:without}.

\begin{table}[ht!]
\caption{Comparison of GMI with EMI and PII, when the attacker does not have any auxiliary knowledge.}
\label{table:without}
\centering
\resizebox{\linewidth}{!}{ 
\begin{tabular}{cccccc}
\toprule
                                                \textbf{ Model }         &\textbf{ Attack   }          & \textbf{KNN Dist} & \textbf{Feat Dist} & \textbf{Attack Acc} & \textbf{Top-5 Attack Acc} \\ \hline
\multicolumn{1}{c}{\multirow{3}{*}{\textbf{VGG16}}}       & \textbf{EMI} & 2397.50           & 2255.54            & 0                   & 0                         \\
\multicolumn{1}{c}{}                                      & \textbf{PII} & 2368.77           & 2425.09            & 0                   & 0                         \\
\multicolumn{1}{c}{}                                      & \textbf{GMI} & \textbf{2098.92}  & \textbf{2012.10}   & \textbf{28}         & \textbf{53}               \\ \hline
\multicolumn{1}{c}{\multirow{3}{*}{\textbf{ResNet-152}}}  & \textbf{EMI} & 2422.99           & 2288.13            & 0                   & 1                         \\
\multicolumn{1}{c}{}                                      & \textbf{PII} & 2368.77           & 2425.09            & 0                   & 0                         \\
\multicolumn{1}{c}{}                                      & \textbf{GMI} & \textbf{1969.09}  & \textbf{1886.44}   & \textbf{44}         & \textbf{72}               \\ \hline
\multicolumn{1}{c}{\multirow{3}{*}{\textbf{face.evolve}}} & \textbf{EMI} & 2371.52           & 2248.81            & 0                   & 1                         \\
\multicolumn{1}{c}{}                                      & \textbf{PII} & 2368.77           & 2425.09            & 0                   & 0                         \\
\multicolumn{1}{c}{}                                      & \textbf{GMI} & \textbf{1923.72}  & \textbf{1802.62}   & \textbf{46}         & \textbf{76}               \\ \bottomrule
\end{tabular}
}
\vspace{-0.5em}
\end{table}

\begin{figure*}[t]
  \centering
  \includegraphics[width=1\textwidth]{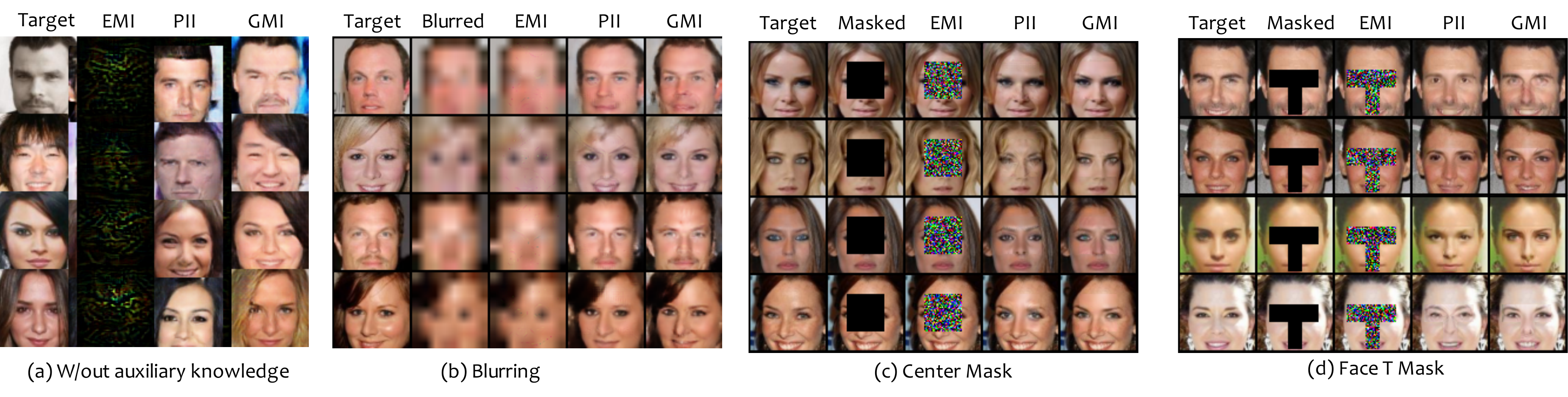}
  \vspace{-2.2em}
  \caption{Qualitative comparison of the proposed GMI attack with the existing MI attack (EMI). When the attacker has access to blurred or corrupted private images as auxiliary knowledge, we additionally compare with the pure image inpainting method (PII). The ground truth target image is shown in 1st col.
}
\vspace{-0.5em}

  \label{fig:img_results}
\end{figure*}

We now discuss the case where the attacker has access to some auxilliary knowledge in terms of blurred or partially blocked images. For the latter, we consider two types of masks---center and face ``T'', illustrated by the second column of Figure~\ref{fig:img_results} (c) and (d), respectively. The center mask blocks the central part of the face and hides most of the identity-revealing features, such as eyes and nose, while the face T mask is designed to obstruct all private features in a face image. EMI takes into account the auxiliary knowledge by using it as a starting point to optimize identity loss. GMI and PII add another branch in the generator to take the auxiliary knowledge as an extra input; the detailed architectures can be found in the supplementary material. Table~\ref{table:method_compare} shows that our method consistently outperforms EMI and PII. In particular, the comparison between GMI and PII indicates that the improved attack performance of GMI over EMI is \emph{not} merely due to the realistic recovery---it truly reveals private information from the target networks. Moreover, the attacks are more effective for the center mask than the face T mask. This is because the face T mask we designed completely hides the identity revealing features on the face while the center mask may still expose the mouth information.


\begin{table*}[ht!]
\caption{Comparison of GMI with EMI and PII, when the attacker can access blurred and corrupted private images. }
\vspace{-0.5em}
\label{table:method_compare}
\centering
\resizebox{0.7\linewidth}{!}{ 
\begin{tabular}{ccccccccccc}
\toprule
                                                   \multirow{2}{*}{\textbf{Model}}      &           \multirow{2}{*}{\textbf{Metric}}           & \multicolumn{3}{c}{\textbf{Blurring}}                                                                  & \multicolumn{3}{c}{\textbf{Center Mask}}      & \multicolumn{3}{c}{\textbf{Face T mask}}      \\ \cline{3-11} 
                                                           &                     & \multicolumn{1}{c}{\textbf{EMI}} & \multicolumn{1}{c}{\textbf{PII}} & \multicolumn{1}{c}{\textbf{GMI}} & \textbf{EMI} & \textbf{PII} & \textbf{GMI}     & \textbf{EMI} & \textbf{PII} & \textbf{GMI}     \\ \midrule
\multicolumn{1}{c}{\multirow{4}{*}{\textbf{VGG16}}}       & \textbf{PSNR}       & 19.66                            & 20.78                            & \textbf{21.97}                    & 18.69        & 25.49        & \textbf{27.58}   & 19.77        & 24.05        & \textbf{26.79}   \\
\multicolumn{1}{c}{}                                      & \textbf{Feat Dist}  & 2073.56                          & 2042.99                          & \textbf{1904.56}                  & 1651.72      & 1866.07      & \textbf{1379.26} & 1798.85      & 1838.31      & \textbf{1655.35} \\
\multicolumn{1}{c}{}                                      & \textbf{KNN Dist}   & 2164.40                          & 2109.82                          & \textbf{1946.97}                  & 1871.21      & 1772.74      & \textbf{1414.37} & 1980.68      & 1916.67      & \textbf{1742.74} \\
\multicolumn{1}{c}{}                                      & \textbf{Attack Acc} & 0\%                              & 6\%                              & \textbf{43\%}                     & 14\%         & 34\%         & \textbf{78\%}    & 11\%         & 20\%         & \textbf{58\%}    \\ \midrule
\multicolumn{1}{c}{\multirow{4}{*}{\textbf{ResNet-152}}}  & \textbf{PSNR}       & 19.63                            & 20.78                            & \textbf{22.00}                    & 18.69        & 25.49        & \textbf{27.34}   & 19.89        & 24.05        & \textbf{26.64}   \\
\multicolumn{1}{c}{}                                      & \textbf{Feat Dist}  & 2006.46                          & 2042.99                          & \textbf{1899.79}                  & 1635.03      & 1866.07      & \textbf{1375.36} & 1641.31      & 1838.31      & \textbf{1594.81} \\
\multicolumn{1}{c}{}                                      & \textbf{KNN Dist}   & 2101.13                          & 2109.82                          & \textbf{1922.14}                  & 1859.78      & 1772.74      & \textbf{1403.24} & 1847.74      & 1916.67      & \textbf{1670.05} \\
\multicolumn{1}{c}{}                                      & \textbf{Attack Acc} & 1\%                              & 6\%                              & \textbf{50\%}                     & 9\%          & 34\%         & \textbf{80\%}    & 11\%         & 20\%         & \textbf{63\%}    \\ \midrule
\multicolumn{1}{c}{\multirow{4}{*}{\textbf{face.evoLVe}}} & \textbf{PSNR}       & 19.64                            & 20.78                            & \textbf{22.04}                    & 18.97        & 25.49        & \textbf{27.69}   & 19.86        & 24.05        & \textbf{25.77}   \\
\multicolumn{1}{c}{}                                      & \textbf{Feat Dist}  & 1997.93                          & 2042.99                          & \textbf{1878.38}                  & 1609.35      & 1866.07      & \textbf{1364.42} & 1762.57      & 1838.31      & \textbf{1624.95} \\
\multicolumn{1}{c}{}                                      & \textbf{KNN Dist}   & 2085.53                          & 2109.82                          & \textbf{1904.47}                  & 1824.10      & 1772.74      & \textbf{1403.19} & 1962.07      & 1916.67      & \textbf{1682.56} \\
\multicolumn{1}{c}{}                                      & \textbf{Attack Acc} & 1\%                              & 6\%                              & \textbf{51\%}                     & 12\%         & 34\%         & \textbf{82\%}    & 11\%         & 20\%         & \textbf{64\%}    \\ \bottomrule
\end{tabular}
}
\vspace{-1.5em}
\end{table*}

Moreover, we examine the performance of the proposed attack for recovering some \emph{implicit} attributes of the private images, such as gender, age, hair style, among others. For some attributes in \texttt{CelebA}, the number of individuals with the attribute is significantly different from that without the attribute. It will be very easy to achieve a high accuracy for recovering these attributes as the attacker can just always output the majority. Therefore, we only focus on some private-sensitive attributes for which \texttt{CelebA} is roughly balanced. Table~\ref{table:attribute} shows that GMI also outperforms EMI in terms of recovering the attributes in various attack settings.

\begin{table}[ht!]
\caption{Comparison of GMI with EMI and PII for recovering implicit attributes of the private images. The attack performance is measured by the accuracy (\%) of a classifier trained to detect a specific attribute in a face image. }
\vspace{-0.5em}
\label{table:attribute}
\centering
\resizebox{\linewidth}{!}{
\begin{tabular}{cccccccc}
\toprule
\multicolumn{1}{c}{\textbf{Setting}} & \multicolumn{1}{c}{\textbf{Attack}} & \textbf{\begin{tabular}[c]{@{}c@{}}Blond\\  Hair\end{tabular}} & \textbf{\begin{tabular}[c]{@{}c@{}}Bushy \\ Eyebrows\end{tabular}} & \textbf{Glasses} & \textbf{Male} & \textbf{Mustache} & \textbf{Young} \\ \hline
\multicolumn{1}{c}{\multirow{3}{*}{\textbf{\begin{tabular}[c]{@{}c@{}}W/out \\ Aux.\\ Knowledge\end{tabular}}}} & \textbf{EMI} & 55 & 65 & 63 & 47 & 74 & 51 \\
\multicolumn{1}{c}{} & \textbf{PII} & 64 & 65 & 78 & 51 & 70 & 61 \\
\multicolumn{1}{c}{} & \textbf{GMI} & \textbf{78} & \textbf{76} & \textbf{90} & \textbf{74} & \textbf{88} & \textbf{70} \\ \hline
\multicolumn{1}{c}{\multirow{3}{*}{\textbf{\begin{tabular}[c]{@{}c@{}}Center \\ Mask\end{tabular}}}} & \textbf{EMI} & 70 & 44 & 67 & 78 & 75 & 84 \\
\multicolumn{1}{c}{} & \textbf{PII} & 76 & 56 & 79 & 75 & 77 & 84 \\
\multicolumn{1}{c}{} & \textbf{GMI} & \textbf{94} & \textbf{79} & \textbf{94} & \textbf{95} & \textbf{92} & \textbf{97} \\ \hline
\multicolumn{1}{c}{\multirow{3}{*}{\textbf{\begin{tabular}[c]{@{}c@{}}Face T \\ Mask\end{tabular}}}} & \textbf{EMI} & 74 & 44 & 55 & 73 & 69 & 77 \\
\multicolumn{1}{c}{} & \textbf{PII} & 80 & 47 & 82 & 70 & 71 & 73 \\
\multicolumn{1}{c}{} & \textbf{GMI} & \textbf{89} & \textbf{71} & \textbf{95} & \textbf{86} & \textbf{90} & \textbf{94} \\ \hline
\multicolumn{1}{c}{\multirow{3}{*}{\textbf{Blurring}}} & \textbf{EMI} & 77 & 67 & 56 & 67 & 75 & 57 \\
\multicolumn{1}{c}{} & \textbf{PII} & 76 & 70 & 77 & 71 & 76 & 65 \\
\multicolumn{1}{c}{} & \textbf{GMI} & \textbf{86} & \textbf{84} & \textbf{92} & \textbf{90} & \textbf{85} & \textbf{82} \\ \bottomrule
\end{tabular}
}
\vspace{-1em}
\end{table}

\vspace{-0.5em}
\subsubsection{Impact of Public Knowledge}
\vspace{-0.5em}
We have seen that distilling prior knowledge and properly incorporating it into the attack algorithm are important to the success of MI attacks. In our proposed method, the prior knowledge is gleaned from public datasets through GAN. We now evaluate the impact of public datasets on the attack performance.

We first consider the case where the public data is from the same distribution as the private data and study how the size of the public data affects the attack performance. We change the size ratio (1:1, 1:4, 1:6, 1:10) of the public over the private data by varying the number of identities in the public dataset (1000, 250, 160, 100). As shown in Table~\ref{table:impact_pub}, the attack performance varies by less than $7\%$ when shrinking the public data size by 10 times.

\begin{table}[ht!]
\caption{Evaluation for the impact of public datasets on the attack accuracy.
}
\vspace{-0.5em}
\label{table:impact_pub}
\centering
\resizebox{\linewidth}{!}{
\begin{tabular}{c|cccc|cc|c}
\toprule
\multirow{2}{*}{\textbf{Model}}    & \multicolumn{4}{c|}{\textbf{CelebA$\rightarrow$CelebA}}    & \multicolumn{2}{c|}{\textbf{PubFig83$\rightarrow$CelebA}} & \multirow{2}{*}{\textbf{EMI}} \\ \cline{2-7}
                     & \textbf{1:1} & \textbf{1:4} & \textbf{1:6} & \textbf{1:10} & \textbf{W/o Preproc.}     & \textbf{W/ Preproc.}     &                                    \\ \hline
\textbf{VGG}         & 78\%         & 77\%         & 75\%         &   72\%            & 48\%                             & 67\%                   & 14\%                               \\
\textbf{LeNet}   & 81\%         & 75\%         & 77\%         &     75\%          & 52\%                             & 66\%                   & 9\%                                \\
\textbf{face.evoLVe} & 77\%         & 77\%         & 77\%         &     70\%        & 56\%                             & 70\%                   & 12\%                               \\ \bottomrule
\end{tabular}
}
\vspace{-1.5em}
\end{table}

Moreover, we study the effect of the distribution shift between the public and private data on the attack performance. We train the GAN on the \texttt{PubFig83} dataset, which contains 13,600 images with 83 identities, and attack the target network trained on \texttt{CelebA}. There are more faces with sunglasses in \texttt{PubFig83} than \texttt{CelebA}, which makes it harder to distill generic face information.
Without any pre-processing, the attack accuracy drops by more than 20\% despite still outperforming the existing MI attack by a large margin. To further improve the reconstruction quality, we detect landmarks in the face images using the off-the-shelf detector\footnote{\url{http://dlib.net/files/shape_predictor_68_face_landmarks.dat.bz2}}, rotate the images such that the eyes lie on a horizontal line, and crop the faces
to remove the background. These pre-processing steps make the public datasets better present the face information, thus improving the attack accuracy significantly.


\begin{figure*}[ht!]
  \centering
  \includegraphics[width=\linewidth]{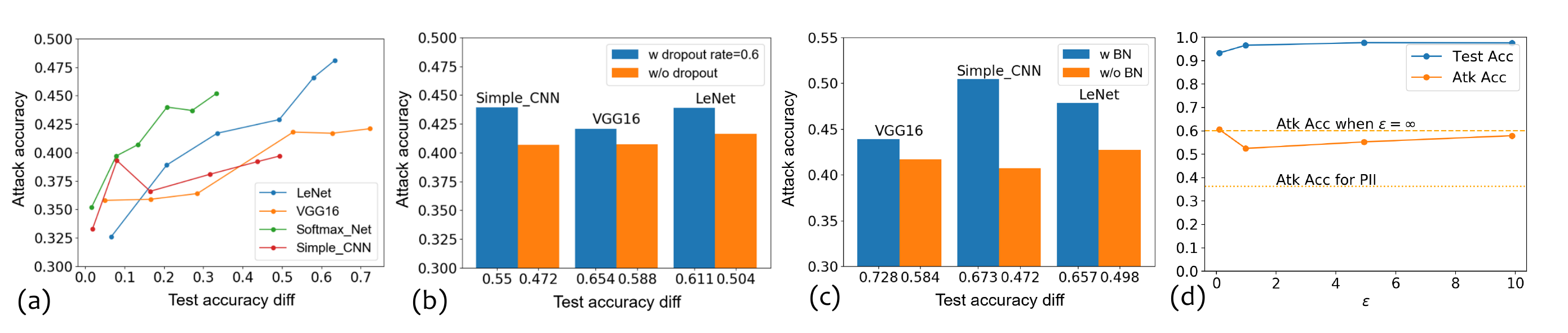}
\vspace{-2.5em}
  \caption{(a)-(c): The performance of the GMI attack against models with different predictive powers by varying training size, dropout, and batch normalization, respectively. (d) Attack accuracy of the GMI attack against models with different DP budgets. Attack accuracy of PII is plotted as a baseline.
}
\vspace{-1.5em}
 \label{fig:connection}
\end{figure*}

\vspace{-0.5em}
\subsubsection{Attacking Models with Varied Predictive Powers}
\vspace{-0.5em}
\label{section:exp_power}

We perform experiments to validate the connection between predictive power and the vulnerability to MI attacks. We measure the predictive power of sensitive feature under a model using the difference of model testing accuracy based on all features and just non-sensitive features. We consider the following different ways to construct models with increasing feature predictive powers, namely, enlarging the training size per class, adding dropout regularization, and performing batch normalization. For the sake of efficiency, we slightly modify the proposed method in Section~\ref{section:pipeline} in order to avert re-training GANs for different architectures. Specifically, we exclude the diversity loss from the attack pipeline so that multiple architectures can share the same GAN for prior knowledge distillation. Figure~\ref{fig:connection} shows that, in general, the attack performance will be better for models with higher feature predictive powers. Moreover, this trend is consistent across different architectures.

\vspace{-0.5em}
\subsubsection{Attacking Differentially Private Models}
\vspace{-0.5em}

We investigate the implications of DP for MI attacks. $(\epsilon,\delta)$-DP is ensured by adding Gaussian noise to clipped gradients in each training iteration~\cite{abadi2016deep}. We find it challenging to produce useful face recognition models with DP guarantees due to the complexity of the task. Therefore, we turn to a simpler dataset, \texttt{MNIST}, which is commonly used in differential private ML studies. We set $\delta=10^{-5}$ and vary the noise scale to obtain target networks with different $\epsilon$. The detailed settings of differentially private training are presented in the supplementary material.
The attack performance against these target networks and their utility are illustrated in Figure~\ref{fig:connection} (d). Since the attack accuracy of the GMI attack on differentially private models is higher than that of PII which fills missing regions completely based on the public data, it is clear that the GMI attack can expose private information from differentially private models, even with stringent privacy guarantees, like $\epsilon=0.1$. Moreover, varying differential privacy budgets helps little to protect against the GMI attack; sometimes, more privacy budgets even improve the attack performance (e.g., changing $\epsilon$ from 1 to 0.1). This is because DP, in its canonical form, only hides the presence of a single instance in the training set; it does not explicitly aim to protect attribute privacy. Limiting the learning of individual training instances may facilitate the learning of generic features of a class, which, in turn, helps to stage MI attacks.

\vspace{-0.5em}
\subsubsection{Results on Other Datasets}
\vspace{-0.5em}


For \texttt{MNIST}, we use all $34265$ images with labels $5,6,7,8,9$ as private set, and the rest of $35725$ images with labels $0,1,2,3,4$ as a public dataset. Note that the labels in the private and public data have no overlaps. We augment the public data by training an autoencoder and interpolating in the latent space. Our GMI attack is compared with EMI in Table~\ref{table:more}. We omit the PII baseline because the public and private set defined in this experiment are rather disparate and PII essentially produces results close to random guesses. We can see from the table that the performance of GMI is significantly better than the EMI. 

Moreover, we attack a disease predictor trained on \texttt{ChestX-ray8}. We use $10000$ images of seven classes as the private data and the other $10000$ of different labels as public data. The GMI and EMI attack are compared in Table~\ref{table:more}. Again, the GMI attack outperforms the EMI attack by a large margin. 

\begin{table}[ht!]
\caption{Comparing the GMI against the EMI attack on \texttt{MNIST} and \texttt{ChestX-ray8}.}
\vspace{-0.5em}
\label{table:more}
\centering
\resizebox{0.8\linewidth}{!}{
\begin{tabular}{lllll}
\toprule
\textbf{Dataset} & \textbf{Attack} & \textbf{KNN Dist} & \textbf{Feat Dist} & \textbf{Attack Acc} \\ \midrule
\multirow{2}{*}{\textbf{MNIST}} & \textbf{EMI} & 31.60 & 82.69 & 40\% \\
 & \textbf{GMI} & \textbf{4.04} & \textbf{16.17} & \textbf{80\%} \\ \midrule
\multirow{2}{*}{\textbf{ChestX-ray8}} & \textbf{EMI} & 130.19 & 155.65 & 14\% \\
 & \textbf{GMI} & \textbf{63.42} & \textbf{93.68} & \textbf{71\%} \\ \bottomrule
\end{tabular}
}
\vspace{-1.5em}
\end{table}

\section{Conclusion}
\vspace{-0.5em}
In this paper, we present a generative approach to MI attacks, which can achieve the-state-of-the-art success rates for attacking the DNNs with high-dimensional input data. The idea of our approach is to extract generic knowledge from public datasets via GAN and use it to regularize the inversion problem. Our experimental results show that our proposed attack is highly performant even when the public datasets (1) do not include the identities that the adversary aims to recover, (2) are unlabeled, (3) have small sizes, and (4) come from a different distribution from the private data. We also provide theoretical analysis showing the fundamental connection between a model's predictive power and its vulnerability to inversion attacks. For future work, we are interested in extending the attack to the black-box setting and studying effective defenses against MI attacks.

\vspace{-0.5em}
\subsubsection*{Acknowledgement}
\vspace{-0.5em}
This work is supported by NSF grant No.1910100, the Center for Long-Term Cybersecurity, and the Berkeley Deep Drive.

\newpage
{\small
\bibliographystyle{ieee_fullname}
\bibliography{ref}
}

\onecolumn

\appendix

\section{Proof of Theorem 1}
\label{appendix:proof}


\begin{Thm}
Let $f_1$ and $f_2$ are two models such that for any fixed label $y\in \mathcal{Y}$, $ U_{f_1}(x_{ns},y)\geq U_{f_2}(x_{ns},y)$. Then, $S_\text{KL}(p(X_s|y,x_{ns})||p_{f_1}(X_s|y,x_{ns})) \geq  S_\text{KL}(p(X_s|y,x_{ns})||p_{f_2}(X_s|y,x_{ns}))$.
\end{Thm}

\begin{proof}
We can expand the KL divergence  $D_\text{KL}(p(X_s|y,x_{ns})||p_{f_1}(X_s|y,x_{ns})$ as follows. 
\begin{align}
    &D_\text{KL}(p(X_s|y,x_{ns})||p_{f_1}(X_s|y,x_{ns})) \\
    &= E_{X\sim p(X_s|y,x_{ns})} [\log p(X_s|y,x_{ns})] -E_{X\sim p(X_s|y,x_{ns})} [\log p_{f_1}(X_s|y,x_{ns})]
\end{align}
Thus,
\begin{align}
    &D_\text{KL}(p(X_s|y,x_{ns})||p_{f_1}(X_s|y,x_{ns})) - D_\text{KL}(p(X_s|y,x_{ns})||p_{f_2}(X_s|y,x_{ns}))\\
    & = E_{X\sim p(X_s|y,x_{ns})} [\log p_{f_2}(X_s|y,x_{ns}) - \log p_{f_1}(X_s|y,x_{ns})]\\
    & = \sum_{x} p(X_s|y,x_{ns}) \big( \log \frac{p_{f_2}(y|X_s,x_{ns})p(X_s|x_{ns})}{p_{f_2}(y|x_{ns})} - \log \frac{p_{f_1}(y|X_s,x_{ns})p(X_s|x_{ns})}{p_{f_1}(y|x_{ns})}\big)\\
    &= \sum_{x} p(X_s|y,x_{ns}) \bigg(\big(\log p_{f_2}(y|X_s,x_{ns})  - \log p_{f_2}(y|x_{ns}) \big)\nonumber\\
    &\quad \quad - \big(\log  p_{f_1}(y|X_s,x_{ns})  - \log p_{f_1}(y|x_{ns})\big)\bigg)\\
    & = U_{f_2}(x_{ns},y) -  U_{f_1}(x_{ns},y) \leq 0
\end{align}
\end{proof}


\section{Experimental Details}

\subsection{Network Architecture}
\label{appendix:arch}
The GAN architectures for the GMI attacks without auxiliary knowledge, with corrupted private image, and with blurred private image, are shown in Figure~\ref{fig:PII}, \ref{fig:PII_corrupted}, and \ref{fig:PII_blurred}, respectively. Moreover, in the experiments, we use the same GAN architectures for the PII baseline and the GMI attacks.

\begin{figure}[ht!]
  \centering
  \includegraphics[width=0.7\textwidth]{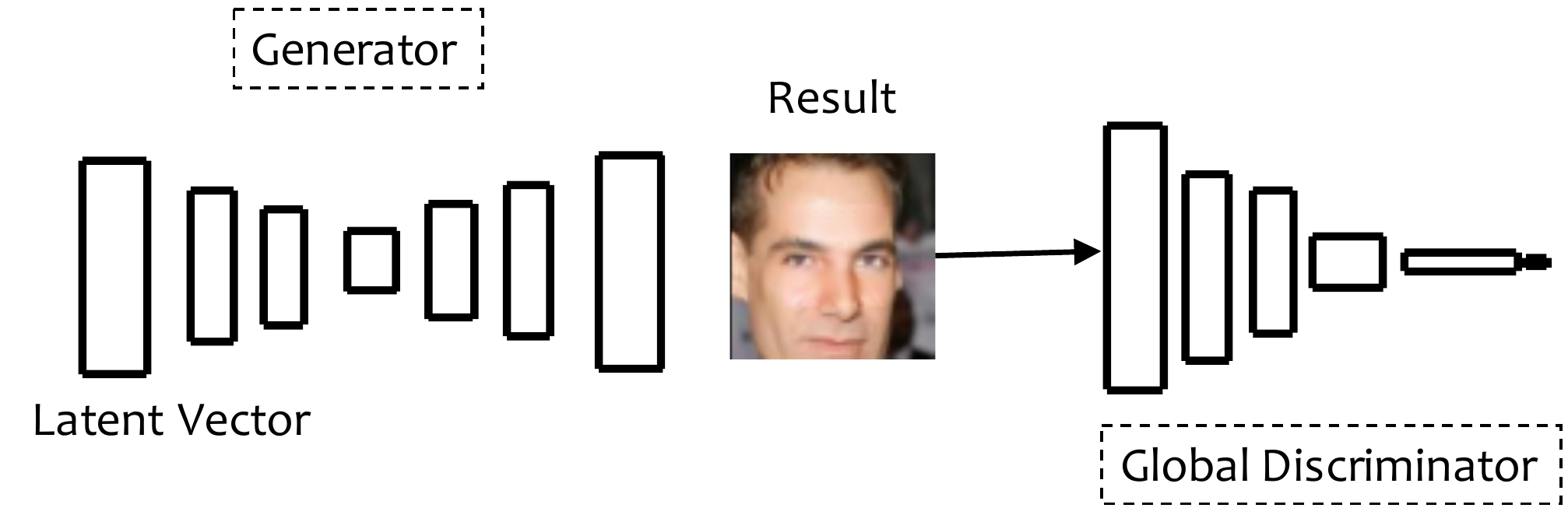}
  \caption{The GAN architecture for the attack without auxiliary knowledge.
}
 \label{fig:PII}
\end{figure}

\begin{figure}[t!]
  \centering
  \includegraphics[width=\textwidth]{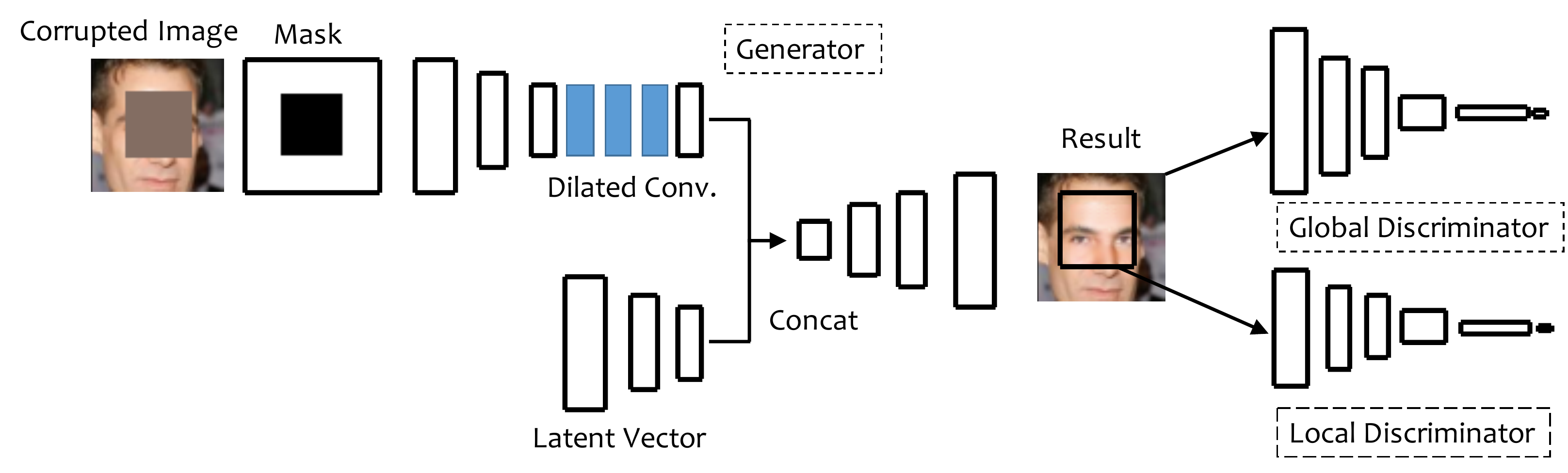}
  \caption{The GAN architecture for the attack with the auxiliary knowledge of a corrupted private image.
}
 \label{fig:PII_corrupted}
\end{figure}

\begin{figure}[t!]
  \centering
  \includegraphics[width=\textwidth]{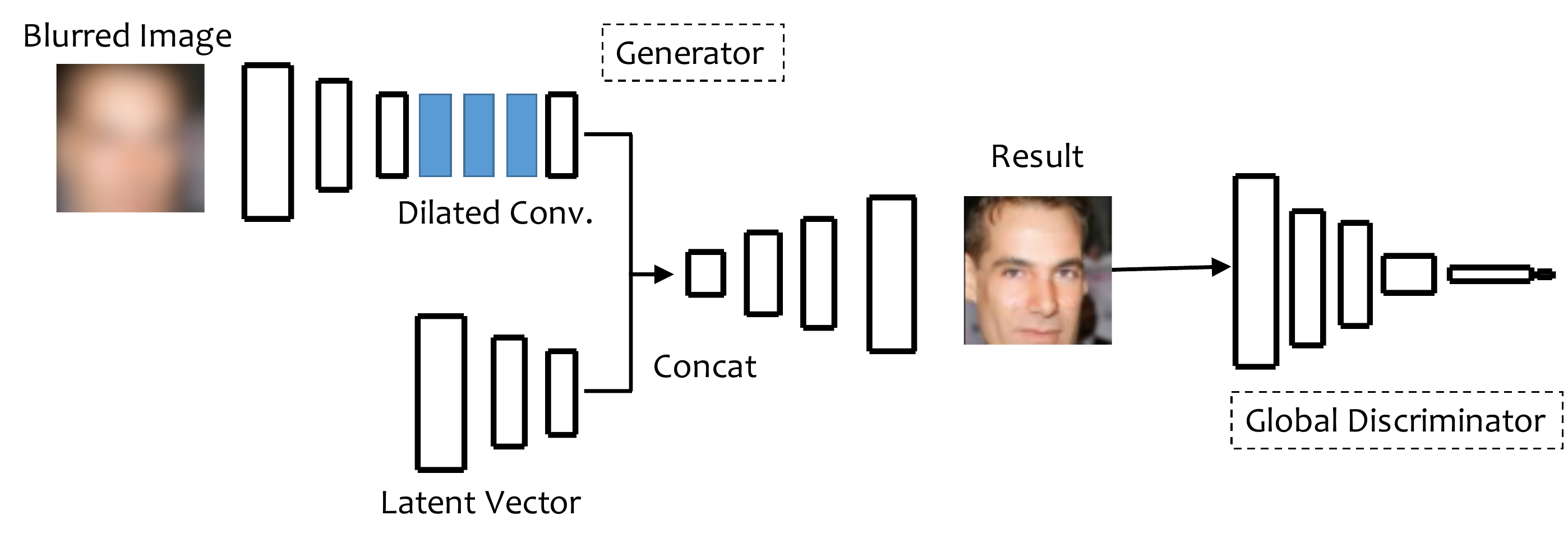}
  \caption{The GAN architecture for the attack with the auxiliary knowledge of a blurred private image.
}
 \label{fig:PII_blurred}
\end{figure}

The detailed architecture designs of the two encoders, the decoder of the generator, the local discriminator, and the global discriminator are presented in Table~\ref{table:revealer_encoder_upper}, Table~\ref{table:revealer_encoder_lower}, Table~\ref{table:revealer_decoder}, Table~\ref{table:global_disc}, and Table~\ref{table:local_disc}, respectively.

\begin{table}[h!]
\begin{center} 
\caption{When the auxiliary knowledge is a corrupted private image, the upper encoder of the generator takes as input the corrputed RGB image and the binary mask. When the auxiliary knowledge is a blurred private image, the upper encoder only takes an image as input. }
\label{table:revealer_encoder_upper}
\begin{tabular}{lllll}
\hline
Type  & Kernel & Dilation & Stride & Outputs \\ \hline
conv. & 5x5    & 1        & 1x1    & 32      \\
conv. & 3x3    & 1        & 2x2    & 64      \\
conv. & 3x3    & 1        & 1x1    & 128     \\
conv. & 3x3    & 1        & 2x2    & 128     \\ \hline
conv. & 3x3    & 1        & 1x1    & 128     \\
conv. & 3x3    & 1        & 1x1    & 128     \\
conv. & 3x3    & 2        & 1x1    & 128     \\
conv. & 3x3    & 4        & 1x1    & 128     \\
conv. & 3x3    & 8        & 1x1    & 128     \\
conv. & 3x3    & 16       & 1x1    & 128     \\ \hline
\end{tabular}
\end{center}
\end{table}

\begin{table}[h!]
\begin{center}
\caption{The lower encoder of the generator that takes as input the latent vector. }
\label{table:revealer_encoder_lower}
\begin{tabular}{llll}
\hline
Type    & Kernel & Stride    & Outputs \\ \hline
linear  &        &           & 8192    \\
deconv. & 5x5    & 1/2 x 1/2 & 256     \\
deconv. & 5x5    & 1/2 x 1/2 & 128     \\ \hline
\end{tabular}
\end{center}
\end{table}

\begin{table}[h!]
\begin{center}
\caption{The decoder of the generator.}
\label{table:revealer_decoder}
\begin{tabular}{llll}
\hline
Type    & Kernel & Stride    & Outputs \\ \hline
deconv. & 5x5    & 1/2 x 1/2 & 128     \\
deconv. & 5x5    & 1/2 x 1/2 & 64      \\
conv.   & 3x3    & 1x1       & 32      \\
conv.   & 3x3    & 1x1       & 3       \\ \hline
\end{tabular}
\end{center}
\end{table}

\begin{table}[h!]
\begin{center}
\caption{The global discriminator.}
\label{table:global_disc}
\begin{tabular}{llll}
\hline
Type  & Kernel & Stride & Outputs \\ \hline
conv. & 5x5    & 2x2    & 64      \\
conv. & 5x5    & 2x2    & 128     \\
conv. & 5x5    & 2x2    & 256     \\
conv. & 5x5    & 2x2    & 512     \\
conv. & 1x1    & 4x4    & 1       \\ \hline
\end{tabular}
\end{center}
\end{table}

\begin{table}[h!]
\begin{center}
\caption{The local discriminator. This discriminator only appears in the attack with the knowledge of a corrupted image.}
\label{table:local_disc}
\begin{tabular}{llll}
\hline
Type  & Kernel & Stride & Outputs \\ \hline
conv. & 5x5    & 2x2    & 64      \\
conv. & 5x5    & 2x2    & 128     \\
conv. & 5x5    & 2x2    & 256     \\
conv. & 1x1    & 4x4    & 1       \\ \hline
\end{tabular}
\end{center}
\end{table}

The information of some network architectures used in the experiment section but not covered in the main text is elaborated as follows:
(1) \texttt{LeNet} adapted from~\cite{Lecun1998Gradient-Based}, which has three convolutional layers, two max pooling layers and one FC layer; (2) \texttt{SimpleCNN}, which has five convolutional layers, each followed by a batch normalization layer and a leaky ReLU layer; (3) \texttt{SoftmaxNet}, which has only one FC layer.





\subsection{The Detailed Setting of the Experiments on ``Attacking Differentially Private Models'' }
\label{appendix:dp}

We split the \texttt{MNIST} dataset into the private set used for training target networks with digits $0\sim 4$ and the public set used for distilling prior knowledge with digits $5\sim 9$. The target network is implemented as a Multilayer Perceptron with 2 hidden layers, which have 512 and 256 neurons, respectively. The evaluation classifier is a convulutional neural network with three convolution layers, followed by two fully-connected layers. It is trained on the entire MNIST training set and can achieve $99.2\%$ accuracy on the \texttt{MNIST} test set.




Differential privacy of target networks is guaranteed by adding Gaussian noise to each stochastic gradient descent step. We use the moment accounting technique to keep track of the privacy budget spent during training~\cite{abadi2016deep}. During the training of the target networks, we set the batch size to be 256. We fix the number of epochs to be 40 and clip the L2 norm of per-sample gradient to be bounded by 1.5. We set the ratio between the noise scale and the gradient clipping threshold to be $0, 0.694, 0.92, 3, 28$, respectively, to obtain the target networks with $\varepsilon = \infty, 9.89, 4.94, 0.98, 0.10 $ when $\delta = 10^{-5}$. For model with $\varepsilon = 0.1$, we use the SGD with a small learning rate 0.01 to ensure stable convergence; otherwise, we set the learning rate to be 0.1.

The architecture of the generator in Section~\ref{appendix:arch} is tailored to the \texttt{MNIST} dataset. We reduce the number of input channels, change the size of kernels, and modify the layers of discriminators to be compatible with the shape of the \texttt{MNIST} data. To train the GAN in the first stage of our GMI attack, we set the batch size to be 64 and use the Adam optimizer with the learning rate 0.004, $\beta_1=0.5$, and $\beta_2=0.999$~\cite{kingma2014adam}. For the second stage, we set the batch size to be 64 and use the SGD with the Nesterov momentum that has the learning rate 0.01 and momentum 0.9. The optimization is performed for 3000 iterations. 

The center mask depicted in the main text is used to block the central part of digits. We report the attack accuracy averaged across 640 randomly sampled images from the private set and 5 random initializations of the latent vector for each sampled image.


\end{document}